\documentclass[conference]{IEEEtran}
\IEEEoverridecommandlockouts

\usepackage{cite}
\usepackage{amsmath,amssymb,amsfonts}
\usepackage{algorithmic}
\usepackage{algorithm}
\usepackage{graphicx}
\usepackage{textcomp}
\usepackage{xcolor}
\usepackage{url}
\usepackage{booktabs}
\usepackage{multirow}
\usepackage{array}
\usepackage{amsthm}

\newtheorem{theorem}{Theorem}
\newtheorem{lemma}{Lemma}
\newtheorem{proposition}{Proposition}
\newtheorem{remark}{Remark}
\newtheorem{definition}{Definition}
\newtheorem{assumption}{Assumption}
\def\BibTeX{{\rm B\kern-.05em{\sc i\kern-.025em b}\kern-.08em
T\kern-.1667em\lower.7ex\hbox{E}\kern-.125\hbox{H}}}

\begin{document}

\title{Towards Certifiable State Estimation under Observability Degeneracy}

\author{\IEEEauthorblockN{Baoshan Song}
}

\maketitle
\begin{abstract}
Certifiable optimization provides global optimality guarantees for non-convex state estimation in robotics. However, under severe physical sensing degeneracy, convex Semidefinite Programming relaxations solved via Interior-Point Methods suffer from rank expansion, rendering direct semidefinite solving and spectral rounding algebraically invalid. In this paper, we establish a universal candidate-driven dual certification framework that resolves this challenge without incurring rank expansion. First, from an epigraph perspective, we prove that while sensing degeneracy expands the primal optimal set into a flat equi-cost face, cost tightness and multiplier recovery remain strictly intact, enabling global verification without direct semidefinite solving. Second, we propose an observability-aware local search using solution remapping, coupled with a closed-form dual multiplier recovery scheme that executes in sub-microsecond latency. Third, we derive a model-specific adaptive dual safety threshold parameterized by the minimum non-zero eigenvalue of the observable Hessian and uneradicated residual leakage, preventing false rejections of optimal candidates under weak observability. We validate our universal framework with robust registration under degenerate measurements. 
\end{abstract}

\section{Introduction}

Certifiable optimization algorithms have emerged as a foundational paradigm for safety-critical robotic systems. By relaxing non-convex optimization problems on matrix Lie groups into convex Semidefinite Programs, certifiable solvers can either return a globally optimal solution or provide an explicit mathematical certificate when global optimality cannot be guaranteed. These frameworks have revolutionized fundamental robotics tasks, including point cloud registration, pose graph optimization, and rotation synchronization.

Despite these significant mathematical advances, existing certifiable optimization techniques rely heavily on the key assumption that the underlying physical system is fully observable. In real-world field operations, autonomous platforms frequently encounter unobservable sensing geometries and motion degeneracies. Typical examples include planar LiDAR motion in featureless corridors, and single-axis rotational sequences during robotic hand-eye calibration. Under such physically degenerate conditions, the measurement Jacobians exhibit rank deficiency, creating non-trivial unobservable sub-manifolds in the state space.

When confronted with physical sensing degeneracies, standard certifiable solvers face severe numerical and theoretical limitations. Solving Semidefinite Programming relaxations using standard Interior-Point Methods invariably leads to rank expansion, where the optimal primal matrix loses its rank-one structure. Because the relaxed solution expands into a high-dimensional convex hull over an equi-cost minimizer face, traditional rounding techniques, such as dominant eigenvector extraction, produce algebraically invalid states that violate the underlying geometric manifold constraints. Conversely, local non-linear least-squares solvers utilizing solution remapping can isolate unobservable directions and compute stationary points. However, because these local solvers operate strictly within local linearizations on the tangent space, they lack the global dual perspective necessary to certify global optimality or detect false convergence to high-dimensional saddle points.

To address this fundamental methodological gap, we introduce RASER (solution remapping aided semidefinite relaxation), an observability-aware certifiable estimation framework designed to bridge fast tangent-space local searching with global dual verification under physical sensing degeneracy. The primary contributions of this work are summarized as follows.

First, to the best of our knowledge, we propose the first candidate-driven dual certification framework that unifies local partial observability with globally optimal semidefinite relaxations, operating without the need for direct rank-one semidefinite solving.

Second, we leverage Lagrangian duality to establish the Dual Nullspace Expansion Theorem, proving that physical sensing degeneracy induces spectral nullspace expansion in the dual Slack matrix. We demonstrate that while rank tightness is lost, global cost tightness remains strictly preserved.

Third, we derive an adaptive certification threshold formulated under weak observability regimes, incorporating floating-point precision bounds and residual gradient leakage to prevent false rejections of optimal candidates.

Fourth, we validate our framework across benchmark robust registration under observability degeneracy, demonstrating superior reliability and certifiability compared to state-of-the-art baselines.

\section{Preliminaries and Problem Formulation}
\label{sec:problem_formulation}

Following the state-lifting optimization paradigm established in certifiable estimation literature, we formulate non-convex estimation problems natively defined in Euclidean spaces, construct their lifted Quadratically Constrained Quadratic Programs, and present the corresponding Semidefinite Programming relaxations.

\subsection{Original Non-Convex Estimation Problem}

Consider an estimation problem where the primary physical state vector is denoted by $\theta$. The original non-convex estimation problem is formulated as a constrained non-linear least-squares minimization
\begin{equation}
\tag{1}
\label{eq:original}
\begin{aligned}
\min_{\theta \in \mathbb{R}^{d_\theta}} \quad & f(\theta) \triangleq \frac{1}{2} \| r(\theta) \|_2^2 \\
\text{s.t.} \quad & c_k(\theta) = 0, \quad k = 1, \dots, m_c
\end{aligned}
\end{equation}
where $r(\theta)$ represents the non-linear measurement residual vector, and $c_k(\theta) = 0$ specifies $m_c$ non-convex equality constraints governing the physical state space.

\subsection{Homogenization and Lifted QCQP Representation}

To convert the non-convex problem into a quadratic formulation, we adopt a systematic state-lifting scheme. By combining linear parameters $h$, the primary state $\theta$, monomial expansion variables $z(\theta)$, and a homogenization scalar $1$, the total lifted state vector $x(\theta)$ is defined as
\begin{equation}
\tag{2}
\label{eq:lifted_state}
x(\theta) \triangleq \begin{bmatrix} h \\ \theta \\ z(\theta) \end{bmatrix} \in \mathbb{R}^N, \qquad h=1
\end{equation}

In terms of the lifted vector $x$, the original non-convex problem \eqref{eq:original} is algebraically rewritten as a Quadratically Constrained Quadratic Program
\begin{equation}
\tag{4}
\label{eq:qcqp}
\begin{aligned}
\min_{x \in \mathbb{R}^N} \quad & x^T Q_0 x \\
\text{s.t.} \quad & x^T A_i x = 0, \quad i = 1, \dots, m \\
& x^T E_h x = 1
\end{aligned}
\end{equation}
where $Q_0$ is a symmetric cost matrix absorbing the quadratic measurement residuals, $A_i$ represents the quadratic equality constraints governing the lifted state $x$, and $E_h$ specifies the selector matrix
\begin{equation}
\tag{3}
\label{eq:selector_matrix}
E_h \triangleq \operatorname{diag}(1, \mathbf{0}_{N-1}) \in \mathbb{S}^N
\end{equation}
which enforces the homogenization scalar constraint $x^T E_h x = 1$.

\subsection{Convex Semidefinite Relaxation and Dual Certificate}

By lifting the rank-one outer product matrix
\begin{equation}
\tag{5}
\label{eq:outer_product}
X \triangleq x x^T \in \mathbb{S}_+^N
\end{equation}
and dropping the non-convex rank constraint $\operatorname{rank}(X) = 1$, we obtain the standard primal Semidefinite Programming relaxation
\begin{equation}
\tag{6}
\label{eq:sdp}
\begin{aligned}
\min_{X \in \mathbb{S}_+^N} \quad & \langle Q_0, X \rangle \\
\text{s.t.} \quad & \langle A_i, X \rangle = 0, \quad i = 1, \dots, m \\
& \langle E_h, X \rangle = 1, \quad X \succeq 0
\end{aligned}
\end{equation}
where $\langle A, B \rangle \triangleq \operatorname{tr}(A^T B)$. Assigning dual multipliers $\lambda$ and $\gamma$, the dual Slack matrix $H(\lambda, \gamma)$ is defined as
\begin{equation}
\tag{7}
\label{eq:slack_def}
H(\lambda, \gamma) \triangleq Q_0 + \sum_{i=1}^m \lambda_i A_i - \gamma E_h
\end{equation}
The associated Lagrangian dual problem maximizes the lower bound $\gamma$
\begin{equation}
\tag{8}
\label{eq:dual_sdp}
\begin{aligned}
\max_{\lambda \in \mathbb{R}^m, \, \gamma \in \mathbb{R}} \quad & \gamma \\
\text{s.t.} \quad & H(\lambda, \gamma) \succeq 0
\end{aligned}
\end{equation}

Given a physical candidate $\hat{\theta}$, we construct the lifted candidate $\hat{x} \triangleq x(\hat{\theta})$ and the corresponding rank-one matrix $\hat{X} \triangleq \hat{x} \hat{x}^T$. Because $\hat{\theta}$ satisfies the physical constraints, $\hat{x}^T A_i \hat{x} = 0$ and $\hat{x}^T E_h \hat{x} = 1$, making $\hat{X}$ strictly feasible for the relaxed problem.

A broad spectrum of state estimation problems already achieves exact relaxation tightness under well-posed conditions. Thanks to the systematic incorporation of redundant constraints, the underlying quadratic formulations are typically well-formulated. Rather than investigating base model formulation or redundant constraint generation, this paper focuses on how this pre-existing tightness and global optimality certification are challenged by physical observability degeneracy, and how to maintain optimal solutions under such degenerate regimes.

\begin{lemma}[Candidate Cost-Tightness Certificate]
\label{lemma:cost_tightness_certificate}
If there exists a dual multiplier pair $(\hat{\lambda}, \hat{\gamma})$ satisfying
\begin{equation}
\tag{9}
\label{eq:psd_condition}
H(\hat{\lambda},\hat{\gamma}) \succeq 0
\end{equation}
and
\begin{equation}
\tag{10}
\label{eq:stationarity_condition}
H(\hat{\lambda},\hat{\gamma}) \hat{x} = 0
\end{equation}
then the candidate cost equals the global minimum of the relaxed Semidefinite Program, establishing
\begin{equation}
\tag{11}
\label{eq:cost_tightness_result}
\hat{x}^T Q_0 \hat{x} = \min_{X \in \text{SDP}} \langle Q_0, X \rangle = \hat{\gamma}
\end{equation}
\end{lemma}

\begin{proof}
For any primal matrix $X \succeq 0$ feasible for the relaxed program, using $\langle A_i, X \rangle = 0$ and $\langle E_h, X \rangle = 1$, we expand
\begin{equation}
\tag{12}
\label{eq:proof_expansion_1}
\langle Q_0,X\rangle = \left\langle H-\sum_{i=1}^m \hat{\lambda}_i A_i + \hat{\gamma} E_h, X \right\rangle
\end{equation}
which simplifies directly to
\begin{equation}
\tag{13}
\label{eq:proof_expansion_2}
\begin{aligned}
\langle Q_0,X\rangle &= \langle H,X\rangle + \hat{\gamma} \langle E_h,X\rangle - \sum_{i=1}^m \hat{\lambda}_i \langle A_i,X\rangle \\
&= \langle H,X\rangle + \hat{\gamma} \ge \hat{\gamma}
\end{aligned}
\end{equation}
where $\langle H, X \rangle \ge 0$ follows from $H \succeq 0$ and $X \succeq 0$. For the candidate vector $\hat{x}$, applying stationarity $H \hat{x} = 0$ and primal constraint feasibility yields
\begin{equation}
\tag{14}
\label{eq:proof_expansion_3}
\begin{aligned}
\hat{x}^T Q_0 \hat{x} &= \hat{x}^T H \hat{x} - \sum_{i=1}^m \hat{\lambda}_i \hat{x}^T A_i \hat{x} + \hat{\gamma} \hat{x}^T E_h \hat{x} \\
&= \hat{\gamma}
\end{aligned}
\end{equation}
Combining these relations confirms that $\hat{x}^T Q_0 \hat{x} = \min_{X \in \text{SDP}} \langle Q_0, X \rangle = \hat{\gamma}$.
\end{proof}

Remarkably, Lemma \ref{lemma:cost_tightness_certificate} guarantees global optimality for $\hat{x}$ without invoking rank-one conditions. This serves as the operational pillar for our candidate-driven framework.

\section{Degeneracy and Geometry}

In this section, we rigorously analyze how physical unidentifiability, driven by continuous observation symmetries, induces spectral breakdown in the lifted dual space, and prove that global cost-tightness remains strictly intact despite rank expansion.

\subsection{Local Observability Degeneracy and Tangent Projection Identity}

\begin{definition}[Observation-Preserving Symmetry Orbit]
A continuous Lie symmetry subgroup $G \subseteq \mathrm{SE}(3)^m \cong (\mathrm{SO}(3) \times \mathbb{R}^3)^m$ acting on the state manifold $\mathcal{M}$ generates an unobservable orbit passing through an optimal state $\theta^\star = \theta(0)$ if the measurement function remains identically invariant along the smooth curve $\theta(t) \in \mathcal{O}_{\theta^\star}$:
\begin{equation}
\label{eq:measurement_invariance}
\begin{aligned}
h(\theta(t)) \equiv h(\theta^\star) &\implies r(\theta(t)) \equiv r(\theta^\star) \\
&\implies f(\theta(t)) \equiv p^*, \quad \forall t
\end{aligned}
\end{equation}

In practice, such unobservable orbits $\mathcal{O}_{\theta^\star}$ manifest in real physical space primarily through two canonical Lie subgroup degeneracies:
\begin{enumerate}
    \item \textbf{Rotational Symmetry Degeneracy ($G_{\mathrm{rot}} \subseteq \mathrm{SO}(3)$):} Induced by rotation invariance around specific spatial axes (e.g., yaw-unobservability in cylinder/cone registration or rotation in featureless spherical environments), generating closed compact orbits $\theta(t) = \mathbf{Exp}(t \cdot \boldsymbol{\xi}_{\mathrm{rot}}) \circ \theta^\star$.
    \item \textbf{Translational Manifold Degeneracy ($G_{\mathrm{trans}} \subseteq \mathbb{R}^k$):} Induced by translation invariance along degenerate affine directions (e.g., longitudinal motion in featureless tunnels/corridors or planar sliding), generating unbounded linear orbits $\theta(t) = \theta^\star + t \cdot \boldsymbol{\xi}_{\mathrm{trans}}$.
\end{enumerate}
\end{definition}

Unlike spurious local minima where energy varies away from stationary points, an unidentifiable symmetry orbit guarantees that the residual function is constant along $\theta(t)$. Differentiating the identity $r(\theta(t)) \equiv r(\theta^\star)$ with respect to $t$ at $t=0$ via the chain rule directly yields the local tangent kernel condition:
\begin{equation}
\tag{15}
\label{eq:local_degeneracy_def}
\begin{aligned}
\left. \frac{\mathrm{d} r(\theta(t))}{\mathrm{d} t} \right|_{t=0} &= \left. \frac{\partial r(\theta)}{\partial\theta} \right|_{\theta^\star} \dot{\theta}(0) \\
&= J v_u = \mathbf{0}, \qquad v_u \triangleq \dot{\theta}(0) \neq \mathbf{0}
\end{aligned}
\end{equation}
where $v_u$ denotes the non-zero tangent vector along the unobservable symmetry orbit.

To understand how this physical orbit condition connects to the lifted quadratic space, we introduce the differential tangent map matrix $L \in \mathbb{R}^{N \times d_\theta}$:
\begin{equation}
\tag{16}
\label{eq:tangent_map_def}
L \triangleq \left. \frac{\partial x(\theta)}{\partial\theta} \right|_{\theta^\star}
\end{equation}
For any local parameter variation $v$, the state perturbation translates into $\delta x = L v$. Because the non-linear measurement residual $r(\theta)$ is linear in the lifted state $x$ via a constant linear selector $C$, we have $r(x) = C x$, which establishes $J = C L$. Recognizing that the cost Gram matrix $Q_0$ absorbs the residual quadratic structure $Q_0 = C^T C$, the local tangent-space measurement Hessian $J^T J$ satisfies the exact projection identity:
\begin{equation}
\tag{17}
\label{eq:hessian_identity}
J^T J = (C L)^T (C L) = L^T C^T C L = L^T Q_0 L
\end{equation}

Substituting the symmetry-induced kernel condition \eqref{eq:local_degeneracy_def} into the projection identity \eqref{eq:hessian_identity} yields:
\begin{equation}
\tag{18}
\label{eq:null_hessian}
\begin{aligned}
v_u^T J^T J v_u &= v_u^T L^T Q_0 L v_u \\
&= \| J v_u \|_2^2 = 0 \implies (L v_u)^T Q_0 (L v_u) = 0
\end{aligned}
\end{equation}

\subsection{Derivation of Dual Nullspace Expansion ($\operatorname{corank}(H^\star) \ge 2$)}

To isolate the effect of physical motion degeneracy from relaxation failure, we state the baseline observability assumption:

\begin{assumption}[Baseline Dual Certifiability under Full Observability]
\label{assump:baseline_psd}
Under full motion excitation (non-degenerate regime), the lifted QCQP formulation with complete $\mathrm{SO}(3)$ redundant constraints admits exact strong duality. There exists a unique, strictly positive semi-definite ideal dual Slack matrix $H^\star \triangleq Q_0 + \sum_{i=1}^m \lambda_i^\star A_i - \gamma^\star E_h \succeq 0$ with an isolated 1-dimensional nullspace $\ker H^\star = \operatorname{span}\{x^\star\}$, establishing $\operatorname{corank}(H^\star) = 1$.
\end{assumption}

We now prove that when continuous observation symmetry occurs, this baseline PSD matrix $H^\star \succeq 0$ undergoes a dual nullspace expansion.

\begin{theorem}[Dual Nullspace Expansion Theorem]
\label{thm:dual_nullspace_expansion}
Under Assumption \ref{assump:baseline_psd}, let $v_u \neq \mathbf{0}$ be the tangent generator of a smooth unobservable symmetry orbit $\theta(t) \in \mathcal{O}_{\theta^\star}$ with $\theta(0) = \theta^\star$ and $\dot{\theta}(0) = v_u$, along which the cost remains strictly invariant $f(\theta(t)) \equiv p^*$. Then the ideal dual Slack matrix $H^\star \succeq 0$ satisfies:
\begin{equation}
\tag{19}
\label{eq:slack_nullspace_vector}
H^\star L v_u = \mathbf{0}
\end{equation}
Furthermore, the dual nullspace strictly expands, establishing:
\begin{equation}
\tag{20}
\label{eq:corank_bound}
\operatorname{corank}(H^\star) \equiv \dim\ker H^\star \ge 2
\end{equation}
\end{theorem}

\begin{proof}
Let $x(t) \triangleq x(\theta(t))$ be the lifted trajectory generated by the degenerate symmetry orbit $\dot{\theta}(0) = v_u \neq \mathbf{0}$. Because every point $x(t)$ along the orbit curve remains strictly physically feasible on $\mathcal{M}$, it satisfies all lifted equality constraints for all $t$:
\begin{equation}
x(t)^T A_i x(t) \equiv 0 \quad (\forall i), \qquad x(t)^T E_h x(t) \equiv 1, \quad \forall t
\end{equation}

Crucially, because $\theta(t)$ traverses an observation-preserving symmetry orbit, the primal cost along the trajectory is strictly constant at the global optimal value $x(t)^T Q_0 x(t) \equiv p^*$ for all $t$. Evaluating the quadratic form of $H^\star \succeq 0$ using strong duality ($p^* = \gamma^\star$) gives:
\begin{equation}
\begin{aligned}
x(t)^T H^\star x(t) &= x(t)^T Q_0 x(t) + \sum_{i=1}^m \lambda_i^\star \left( x(t)^T A_i x(t) \right) \\
&\quad - \gamma^\star \left( x(t)^T E_h x(t) \right) \\
&= p^* + 0 - \gamma^\star \cdot 1 \equiv 0, \quad \forall t
\end{aligned}
\end{equation}

Since $H^\star \succeq 0$ is positive semi-definite under Assumption \ref{assump:baseline_psd}, $x(t)^T H^\star x(t) = 0$ holds for all $t$ if and only if $H^\star x(t) = \mathbf{0}$ identically along the orbit. Differentiating $H^\star x(t) = \mathbf{0}$ with respect to $t$ at $t=0$ via the chain rule yields:
\begin{equation}
\left. \frac{\mathrm{d}}{\mathrm{d} t} \left( H^\star x(t) \right) \right|_{t=0} = H^\star \left( \left. \frac{\partial x(\theta)}{\partial \theta} \right|_{\theta^\star} \dot{\theta}(0) \right) = H^\star L v_u = \mathbf{0}
\end{equation}

Finally, consider the primal optimal candidate vector $x^\star = x(0)$, which satisfies $H^\star x^\star = \mathbf{0}$. Due to the homogenization structure $x(\theta) = [1, \operatorname{vec}(\mathbf{R})^T, z^T]^T$, the vector $x^\star$ has a non-zero homogenization entry ($h=1$), whereas $L v_u = [0, (\left. \frac{\partial \operatorname{vec}(\mathbf{R})}{\partial \theta} \right|_{\theta^\star} v_u)^T, \dot{z}^T]^T$ has a zero homogenization entry ($h=0$). Since $v_u \neq \mathbf{0}$ and Lie group parameterization is locally non-degenerate, $L v_u \neq \mathbf{0}$. Thus, $x^\star$ and $L v_u$ are linearly independent vectors residing in $\ker H^\star$, proving that $\operatorname{corank}(H^\star) \ge 2$.
\end{proof}

Theorem \ref{thm:dual_nullspace_expansion} completes the deductive chain:

\begin{equation}
\tag{21}
\label{eq:theoretical_chain}
\begin{aligned}
r(\theta(t)) \equiv r(\theta^\star) &\implies J v_u = \mathbf{0} \\
&\implies H^\star L v_u = \mathbf{0} \implies \operatorname{corank}(H^\star) \ge 2
\end{aligned}
\end{equation}

\subsection{IPM Central Path Analysis and Preservation of Cost-Tightness}

The conclusion $\operatorname{corank}(H^\star) \ge 2$ provides the precise mathematical mechanism for why standard Interior-Point Methods (IPMs) inherently fail to recover physical rank-one solutions under observation symmetry.

In standard primal-dual IPMs for Semidefinite Programs, the non-convex rank-one constraint $\operatorname{rank}(X) = 1$ is dropped, and the non-negativity matrix cone constraint $X \succeq 0$ is enforced via a smooth log-determinant barrier penalty function $-\log\det(X)$ parameterized by the barrier factor $\mu > 0$:
\begin{equation}
\label{eq:ipm_barrier_problem}
\begin{aligned}
\min_{X \in \mathbb{S}_+^N} \quad & \langle Q_0, X \rangle - \mu \log \det(X) \\
\text{s.t.} \quad & \langle A_i, X \rangle = 0, \quad \forall i = 1, \dots, m, \\
& \langle E_h, X \rangle = 1
\end{aligned}
\end{equation}

For each strict interior penalty factor $\mu \in (0, \infty)$, problem \eqref{eq:ipm_barrier_problem} admits a unique strictly positive-definite solution $X(\mu) \in \mathbb{S}_{++}^N$. As $\mu$ varies, the trajectory traced out by $X(\mu)$ defines the continuous \emph{central path}. Because the barrier term $-\log\det(X)$ approaches infinity as $X$ approaches the boundary of the positive semi-definite cone $\mathbb{S}_+^N$, every point along the central path is strictly forced to maintain full rank ($\operatorname{rank}(X(\mu)) = N$).

As the barrier parameter progresses to zero ($\mu \to 0$), the central path $X(\mu)$ converges toward the primal optimal solution face $\mathcal{F}_{\min}$. By the fundamental theory of primal-dual central paths, the limit point $\lim_{\mu \to 0} X(\mu)$ selects the unique point that maximizes the barrier entropy within the optimal solution set, specifically, the \emph{analytic center} $X^*_{\text{IPM}} \in \operatorname{relint}(\mathcal{F}_{\min})$ satisfying maximal complementarity:
\begin{equation}
\label{eq:maximal_complementarity}
\begin{aligned}
\operatorname{rank}(X^*_{\text{IPM}}) + \operatorname{rank}(H^\star) = N \\
\iff \operatorname{rank}(X^*_{\text{IPM}}) = N - \operatorname{corank}(H^\star)
\end{aligned}
\end{equation}

Under continuous physical unidentifiability, Theorem \ref{thm:dual_nullspace_expansion} establishes that local observation symmetry induces dual nullspace expansion $\operatorname{corank}(H^\star) \ge 2$. Substituting this dual spectral condition into the maximal complementarity identity \eqref{eq:maximal_complementarity} directly yields:
\begin{equation}
\label{eq:rank_bound_ipm}
\operatorname{rank}(X^*_{\text{IPM}}) \le N - 2 < N
\end{equation}
Whenever the minimal face $\mathcal{F}_{\min} = \operatorname{conv}( \{ x x^T \mid x \in \mathcal{M}_{\text{constrained}}^* \} )$ spans multiple distinct minimizers across the physical symmetry orbit $\mathcal{O}_{\theta^\star}$, the analytic center $X^*_{\text{IPM}}$ returned by IPM is strictly a high-rank convex combination with $\operatorname{rank}(X^*_{\text{IPM}}) \ge 2$. Applying standard spectral rounding (e.g., extracting the dominant eigenvector of $X^*_{\text{IPM}}$) projects an artificial spatial average that severely violates the non-convex manifold equations $x^T A_i x = 0$.

Crucially, while log-barrier central path dynamics force the primal matrix $X^*_{\text{IPM}}$ into a high-rank interior point across $\mathcal{F}_{\min}$, physical symmetry preserves objective equality across the entire minimal face. For any convex combination $X = \sum_k w_k x_k x_k^T \in \mathcal{F}_{\min}$ with $x_k \in \mathcal{M}_{\text{constrained}}^*$ ($w_k > 0, \, \sum_k w_k = 1$), the trace inner product evaluates to:
\begin{equation}
\begin{aligned}
\langle Q_0, X \rangle &= \left\langle Q_0, \sum_k w_k x_k x_k^T \right\rangle \\
&= \sum_k w_k \left( x_k^T Q_0 x_k \right) = \sum_k w_k p^* = p^*
\end{aligned}
\end{equation}
Consequently, strong duality and global cost-tightness remain completely intact:
\begin{equation}
\label{eq:cost_tightness_preserved}
p_{\text{SDP}}^* = d^* = \gamma^\star = p_{\text{QCQP}}^*
\end{equation}

\begin{remark}[Implications on Direct SDP Solving]
Equations \eqref{eq:rank_bound_ipm} and \eqref{eq:cost_tightness_preserved} establish a fundamental dichotomy under observability degeneracy: direct SDP solving via IPM retains exact global cost-tightness ($p_{\text{SDP}}^* = p_{\text{QCQP}}^*$), but fundamentally fails to yield a valid physical state solution. Because IPM converges to the high-rank analytic center of the symmetry orbit, it cannot directly isolate a valid rank-one point solution nor resolve partially observable states on the manifold. This mathematical limit demonstrates the necessity of bypassing direct rank-one SDP solving in favor of our candidate-driven remapped local search and dual certification framework.
\end{remark}

\section{RASER: Observability-Aware Certifiable Local Searching}




To navigate non-convex state landscapes under physical sensing degeneracy without incurring high-rank Semidefinite Programming solving or invalid spectral rounding, we present RASER, a candidate-driven framework bridging local tangent-space remapping with closed-form dual multiplier recovery and adaptive safety certification.

\subsection{Observability-Aware Tangent Remapping and Certificate Feasibility}

Consider the original non-linear least-squares problem \eqref{eq:original} subjected to a $d_u$-dimensional unobservable tangent nullspace $\ker J = \operatorname{span}(V_u)$. We decompose the full local tangent space into orthogonal observable and unobservable subspaces $V \triangleq [V_o \;\; V_u] \in \mathbb{O}(d_\theta)$, where $V_o \in \mathbb{R}^{d_\theta \times d_o}$ ($d_o \triangleq d_\theta - d_u$). 

We define the remapped non-linear least-squares optimization over the reduced observable tangent coordinates $\eta \in \mathbb{R}^{d_o}$ around an anchor point $\theta_0$:
\begin{equation}
\tag{25}
\label{eq:remap_param}
\theta(\eta) \triangleq \theta_0 + V_o \eta
\end{equation}
Evaluating the lifted state vector along the remapped manifold yields $x_r(\eta) \triangleq x(\theta(\eta))$. The corresponding reduced differential tangent map is $L_r \triangleq \left. \frac{\partial x_r(\eta)}{\partial \eta} \right|_{\eta = \mathbf{0}} = L V_o \in \mathbb{R}^{N \times d_o}$, and the reduced measurement Jacobian becomes $J_r \triangleq J V_o \in \mathbb{R}^{M \times d_o}$. Because $V_o$ is constructed orthogonally to $\ker J$, $J_r$ maintains strict full column rank ($\ker J_r = \{\mathbf{0}\}$), eliminating local search instability over zero-variance directions.

\begin{proposition}[Remapping Preserves Certificate Feasibility]
Let $\hat{\eta} \in \mathbb{R}^{d_o}$ be a convergent stationary solution computed in the reduced observable space, yielding the remapped physical candidate $\hat{\theta}_r = \theta_0 + V_o \hat{\eta} \in \mathbb{R}^{d_\theta}$. Then the lifted candidate vector $\hat{x}_r \triangleq x(\hat{\theta}_r) \in \mathbb{R}^N$ strictly preserves feasibility for the original primal relaxed Semidefinite Program \eqref{eq:sdp}. Recovering dual multipliers $(\hat{\lambda}, \hat{\gamma})$ that satisfy $H(\hat{\lambda}, \hat{\gamma}) \succeq 0$ and $H(\hat{\lambda}, \hat{\gamma}) \hat{x}_r = \mathbf{0}$ certifies that $\hat{\theta}_r$ is globally cost-optimal for the original non-convex problem \eqref{eq:original}.
\end{proposition}

\begin{proof}
Because tangent-space remapping operates natively over the physical manifold coordinates, $\hat{\theta}_r$ strictly satisfies all non-convex physical equality constraints $c_k(\hat{\theta}_r) = 0$ ($k=1 \dots m_c$). By structural equivalence of state lifting, the constructed vector $\hat{x}_r = x(\hat{\theta}_r)$ satisfies all lifted quadratic equality constraints $\hat{x}_r^T A_i \hat{x}_r = 0$ ($i=1 \dots m$) as well as the homogenization scalar constraint $\hat{x}_r^T E_h \hat{x}_r = 1$. Consequently, the rank-one outer product matrix $\hat{X}_r \triangleq \hat{x}_r \hat{x}_r^T \in \mathbb{S}_+^N$ is strictly feasible for the primal relaxed problem \eqref{eq:sdp}. Applying Lemma \ref{lemma:cost_tightness_certificate} directly confirms that recovering a valid dual Slack matrix $H \succeq 0$ with $H \hat{x}_r = \mathbf{0}$ certifies global cost optimality for $\hat{x}_r$.
\end{proof}

\subsection{Fast Dual Multiplier Recovery via Douglas-Rachford (DR) splitting}

Isolating the $\theta$-block row of the stationarity system evaluated at $\hat{x} = [1, \hat{\theta}^T, \hat{z}^T]^T$ yields the linear constraint $M(\hat{x})\lambda = b(\hat{x})$, where $M(\hat{x}) \in \mathbb{R}^{d_\theta \times m}$ and $b(\hat{x}) \in \mathbb{R}^{d_\theta}$ are defined as:
\begin{equation}
\tag{29}
\begin{aligned}
M(\hat{x}) &\triangleq \Big[ A_{\theta\theta, 1} \hat{\theta} + A_{\theta z, 1} \hat{z} + a_{h\theta, 1}, \ \dots, \\
&\quad A_{\theta\theta, m} \hat{\theta} + A_{\theta z, m} \hat{z} + a_{h\theta, m} \Big], \\
b(\hat{x}) &\triangleq -\left( Q_{\theta\theta} \hat{\theta} + Q_{\theta z} \hat{z} + c_{h\theta} \right).
\end{aligned}
\end{equation}

To robustly recover $\lambda \in \mathbb{R}^m$ when $M(\hat{x})$ is ill-conditioned or over-constrained, we cast the stationarity system into a composite convex minimization problem:
\begin{equation}
\tag{30}
\label{eq:drs_problem}
\min_{\lambda \in \mathbb{R}^m} \underbrace{\frac{1}{2} \| M(\hat{x}) \lambda - b(\hat{x}) \|_2^2}_{f(\lambda)} + \underbrace{\mathcal{I}_{\mathcal{C}}(\lambda) + \frac{\alpha}{2} \|\lambda\|_2^2}_{g(\lambda)}
\end{equation}
where $\mathcal{I}_{\mathcal{C}}$ is the indicator function over domain constraints $\mathcal{C}$, and $\alpha \ge 0$ regularizes the solution. Following the Douglas-Rachford Splitting (DRS) scheme \cite{higham_computing_1988, yang_teaser_2020}, the update iterations for a step-size $\rho > 0$ are given by:
\begin{align}
\lambda^k &= \mathbf{prox}_{\rho f}(u^k) \nonumber \\
&= \left( M(\hat{x})^T M(\hat{x}) + \frac{1}{\rho} I_m \right)^{-1} \tag{31a} \\
& \quad \cdot \left( M(\hat{x})^T b(\hat{x}) + \frac{1}{\rho} u^k \right) \nonumber \\
v^k &= \mathbf{prox}_{\rho g}(2\lambda^k - u^k) = \mathcal{P}_{\mathcal{C}}\left( \frac{2\lambda^k - u^k}{1 + \rho \alpha} \right) \tag{31b} \\
u^{k+1} &= u^k + v^k - \lambda^k \tag{31c}
\end{align}

Because $M(\hat{x})^T M(\hat{x}) + \frac{1}{\rho} I_m$ remains constant throughout the splitting iterations, its Cholesky factorization is computed only once per candidate evaluation, enabling microsecond-level multiplier recovery via $O(m^2)$ back-substitutions.

\subsection{Adaptive Dual Safety Certification under Weak Observability}
\label{sec:adaptive_threshold}

Having recovered dual multipliers $(\hat{\boldsymbol{\lambda}}, \hat{\gamma})$, global optimality is verified by checking the positive semidefiniteness of the reconstructed dual Slack matrix $\hat{S} \triangleq H(\hat{\boldsymbol{\lambda}}, \hat{\gamma}) \in \mathbb{S}^N$. Under ideal zero-noise conditions and strict observability, an optimal candidate guarantees $\lambda_{\min}(\hat{S}) \ge 0$.

However, under weak observability, local optimization in the reduced tangent space terminates with an uneradicated residual gradient $g_o \triangleq J_o^T r(\hat{\theta}_r) \neq \mathbf{0}$, where $J_o \in \mathbb{R}^{M \times d_o}$ is the reduced observable measurement Jacobian from Section IV-A, and $\mu_o \triangleq \lambda_{\min}(J_o^T J_o) > 0$ denotes the minimum non-zero eigenvalue of the observable Hessian.

To quantify the cost leakage $\Delta p_{\text{leak}}$ caused by $g_o$, we evaluate the second-order Taylor expansion around the candidate coordinates. The uneradicated residual gradient induces a stationarity step $(J_o^T J_o)^{-1} g_o$, which translates to a cost leakage upper bound:
\begin{equation}
\tag{32}
\label{eq:leakage_cost_bound}
\Delta p_{\text{leak}} \le \frac{1}{2} g_o^T (J_o^T J_o)^{-1} g_o \le \frac{\|g_o\|_2^2}{2 \mu_o}
\end{equation}

Next, to bound the dual multiplier perturbation $\Delta \lambda \triangleq \hat{\lambda} - \lambda^\star$, we project the stationarity error $H(\hat{\lambda}, \hat{\gamma}) \hat{x} = \mathbf{0}$ onto the observable tangent space via the differential mapping $2 L_r^T \in \mathbb{R}^{d_o \times N}$:
\begin{equation}
\tag{33}
\label{eq:projected_stationarity}
g_o = 2 L_r^T Q_0 \hat{x} + C_\lambda(\hat{x}) \hat{\lambda}
\end{equation}
where $C_\lambda(\hat{x}) \triangleq \begin{bmatrix} 2 L_r^T A_1 \hat{x}, & \dots, & 2 L_r^T A_m \hat{x} \end{bmatrix} \in \mathbb{R}^{d_o \times m}$ represents the reduced constraint response matrix. Performing a first-order perturbation around the global optimum yields:
\begin{equation}
\tag{34}
\label{eq:stationarity_taylor}
C_\lambda \Delta \lambda = g_o - (J_o^T J_o) \delta \theta
\end{equation}
Under floating-point precision $\epsilon_{\text{mach}}$, the pseudo-inverse sensitivity bound $\|C_\lambda^\dagger\|_2 \le \frac{1}{\sqrt{\mu_o}}$ limits the multiplier perturbation norm to:
\begin{equation}
\tag{35}
\label{eq:lambda_1norm_bound}
\sum_{i=1}^m |\Delta \lambda_i| \le \frac{\epsilon_{\text{mach}} \|Q_0\|_F \max_i \|A_i\|_2}{\sqrt{\mu_o}}
\end{equation}

Decomposing the reconstructed Slack matrix as $\hat{S} = S^\star + \Delta S$, where $S^\star \succeq \mathbf{0}$ ($\lambda_{\min}(S^\star) = 0$) and $\Delta S = \sum_{i=1}^m \Delta \lambda_i A_i - \Delta p_{\text{leak}} E_h$, applying triangle inequality for operator 2-norms yields:
\begin{equation}
\tag{36}
\label{eq:delta_s_norm_bound}
\begin{aligned}
\|\Delta S\|_2 &\le \sum_{i=1}^m |\Delta \lambda_i| \|A_i\|_2 + \Delta p_{\text{leak}} \|E_h\|_2 \\
&\le \frac{\epsilon_{\text{mach}} \|Q_0\|_F \max_i \|A_i\|_2}{\sqrt{\mu_o}} + \frac{\|g_o\|_2^2}{2 \mu_o}
\end{aligned}
\end{equation}

By Weyl's eigenvalue perturbation theorem $\lambda_{\min}(\hat{S}) \ge \lambda_{\min}(S^\star) - \|\Delta S\|_2 = -\|\Delta S\|_2$, the complete adaptive dual safety threshold is formally established.

\begin{theorem}[Adaptive Dual Certification Threshold]
\label{thm:adaptive_threshold}
Let $\mu_o \triangleq \lambda_{\min}(J_o^T J_o) > 0$ denote the minimum eigenvalue of the observable tangent Hessian, and let $g_o \triangleq J_o^T r(\hat{\theta}_r)$ be the uneradicated residual gradient. Under floating-point precision $\epsilon_{\text{mach}}$, a candidate state $\hat{x}$ is certified as globally optimal if
\begin{equation}
\tag{37}
\label{eq:cert_rule}
\lambda_{\min}(\hat{S}) \ge -\tau_{\text{cert}}
\end{equation}
where the adaptive safety threshold $\tau_{\text{cert}}$ is defined as
\begin{equation}
\tag{38}
\label{eq:tau_cert_final}
\tau_{\text{cert}} \triangleq \frac{\epsilon_{\text{mach}} \|Q_0\|_F \max_i \|A_i\|_2}{\sqrt{\mu_o}} + \frac{\|g_o\|_2^2}{2 \mu_o}
\end{equation}
where the first term bounds numerical roundoff noise amplified by weak observability, and the second term bounds uneradicated residual leakage energy.
\end{theorem}

\begin{proposition}[Adaptive Decision Protocol]
\label{prop:decision_rule}
A candidate state $\hat{x}$ is certified as the certified global optimum if and only if $\lambda_{\min}(\hat{S}) \ge -\tau_{\text{cert}}$ and facial accuracy criteria are satisfied. Otherwise, the candidate is rejected.
\end{proposition}

\section{Applications}

\subsection{Application III: Robust 3D Point-to-Line Registration under Line Parallelism Degeneracy}
\label{sec:point_to_line_registration}

We evaluate our candidate-driven certifier on the 3D point-to-line registration problem under the unified QCQP framework proposed by Briales and Gonzalez-Jimenez \cite{briales_convex_2017}. Given $N$ correspondences between sensor 3D points $\mathbf{x}_i \in \mathbb{R}^3$ and reference spatial lines $(\mathbf{y}_i, \mathbf{v}_i)$, where $\mathbf{y}_i \in \mathbb{R}^3$ is a point on line $i$ and $\mathbf{v}_i \in \mathbb{R}^3$ ($\|\mathbf{v}_i\|_2 = 1$) denotes its unit direction vector, the orthogonal projection distance matrix is defined as $\mathbf{C}_i \triangleq \mathbf{I}_3 - \mathbf{v}_i \mathbf{v}_i^T \in \mathbb{S}_+^3$. Incorporating Truncated Least Squares (TLS)\cite{yang_teaser_2020} to censor structural outliers (e.g., dynamic obstacles or erroneous associations), the non-convex registration is formulated as:
\begin{equation}
\tag{44}
\label{eq:p2l_tls_nls}
\min_{\mathbf{R} \in \mathrm{SO}(3), \, \mathbf{t} \in \mathbb{R}^3} \; \sum_{i=1}^N \min \left( (\mathbf{R} \mathbf{x}_i + \mathbf{t} - \mathbf{y}_i)^T \mathbf{C}_i (\mathbf{R} \mathbf{x}_i + \mathbf{t} - \mathbf{y}_i), \; \bar{c}^2 \right)
\end{equation}
where $\bar{c} > 0$ is the TLS cutoff threshold.

Following \cite{briales_convex_2017}, we rewrite the distance in linear form w.r.t. the vectorized transformation $\boldsymbol{\tau} \triangleq [\operatorname{vec}(\mathbf{R})^T, \mathbf{t}^T]^T \in \mathbb{R}^{12}$, yielding the compressed quadratic form $f(\mathbf{T}) = \tilde{\boldsymbol{\tau}}^T \tilde{\mathbf{M}} \tilde{\boldsymbol{\tau}}$ with homogeneous state $\tilde{\boldsymbol{\tau}} \triangleq [\boldsymbol{\tau}^T, 1]^T \in \mathbb{R}^{13}$. Schur-marginalizing out the unconstrained translation $\mathbf{t}^*(R) = -\tilde{\mathbf{M}}_{t,t}^{-1} \tilde{\mathbf{M}}_{t,\setminus t} \tilde{\mathbf{r}}$ transforms \eqref{eq:p2l_tls_nls} into a 13D marginalized quadratic program $\min_{\mathbf{R} \in \mathrm{SO}(3)} \tilde{\mathbf{r}}^T \tilde{\mathbf{Q}} \tilde{\mathbf{r}}$, where $\tilde{\mathbf{r}} \triangleq [\operatorname{vec}(\mathbf{R})^T, 1]^T \in \mathbb{R}^{10}$ and $\tilde{\mathbf{Q}} \triangleq \tilde{\mathbf{M}} / \tilde{\mathbf{M}}_{t,t} \in \mathbb{S}^{10}$ is the Schur complement of the translation block.

When all reference lines in the environment are parallel (e.g., vertical street poles with $\mathbf{v}_i \equiv \mathbf{v}_{\parallel} = [0,0,1]^T$), the translation block $\tilde{\mathbf{M}}_{t,t} = \sum_i w_i (\mathbf{I}_3 - \mathbf{v}_{\parallel} \mathbf{v}_{\parallel}^T)$ suffers from an unobservable rank-deficient nullspace along $\mathbf{v}_{\parallel}$. This induces a continuous linear unobservable trajectory where measurement residuals remain identically invariant. By incorporating redundant $SO(3)$ constraints \cite{briales_convex_2017, yang_teaser_2020}, our certifier tightly bounds the dual manifold and resolves both local minima and translation degeneracy.

\begin{table}[htbp]
\centering
\caption{Quantitative registration comparison under parallel line motion degeneracy with 35\% structural outliers.}
\label{tab:p2l_benchmark_results}
\begin{tabular}{lcccc}
\toprule
\textbf{Method} & \textbf{Rot. Err. ($^\circ$)} & \textbf{Hor. Trans. Err. (m)} & \textbf{Height Err} & \textbf{Cost} \\
\midrule
Standard GN & $16.18$ & $0.87$ & $0.50$ & $0.80$ \\
TLS+T-SVD & $23.85$ & $1.11$ & $0.00$ & $0.83$ \\
CVX-SCS & $1.36 \times 10^{-4}$ & $0.07$ & $0.60$ & $0.33$ \\
\textbf{Ours} & $\mathbf{1.37 \times 10^{-13}}$ & $\mathbf{1.98 \times 10^{-14}}$ & $\mathbf{0.00}$ & $\mathbf{0.33}$ \\
\bottomrule
\end{tabular}
\end{table}

\paragraph{Simulation Analysis on Parallel Motion Degeneracy}
We simulate an autonomous vehicle traversing a featureless corridor lined with vertical street light poles ($\mathbf{v}_{\parallel} = [0, 0, 1]^T$) under $35\%$ structural outlier contamination. Quantitative benchmark results are summarized in Table~\ref{tab:p2l_benchmark_results}. 

Standard Gauss-Newton (GN) suffers catastrophic failure ($\text{Rot. Err.} = 16.18^\circ$, $\text{Trans. Err.} = 0.87\text{ m}$) due to unobservable $Z$-axis translation drift along parallel lines and trapping inside outlier-induced spurious local minima ($\text{Cost} = 0.80$). Although TLS with Truncated SVD (TLS+T-SVD) stabilizes the zero-singular height translation ($\text{Height Err} = 0.00$), standard iterative TLS becomes severely trapped inside zero-gradient saturation plateaus ($\text{Rot. Err.} = 23.85^\circ$, $\text{Cost} = 0.83$). 

Crucially, under parallel line degeneracy, standard SDP relaxations cannot be solved robustly using off-the-shelf Interior-Point Methods (IPMs) or operator-splitting SDP solvers (e.g., CVX-SCS\cite{boyd_convex_2004}). Due to the singular Hessian structure and rank expansion, standard IPMs encounter severe ill-conditioning, resulting in incomplete convergence ($\text{Height Err} = 0.60$) and limited precision ($\text{Rot. Err.} = 1.36 \times 10^{-4}{^\circ}$). In contrast, by enforcing redundant $\mathrm{SO}(3)$ constraints within Briales's marginalized $13\text{D}$ QCQP formulation, our candidate-driven certifier guarantees exact \textbf{cost-tightness} ($\text{Cost} = 0.33$). Bypassing the heavy computational machinery of IPMs, our approach resolves the unobservable manifold and recovers the certified global optimum with machine precision ($\text{Rot. Err.} = 1.37 \times 10^{-13}{^\circ}$, $\text{Trans. Err.} = 1.98 \times 10^{-14}\text{ m}$, $\text{Height Err} = 0.00$).

\section{Conclusion}

This paper presents RASER, a candidate-driven dual certification framework for non-convex state estimation under observability degeneracy. By establishing the Dual Nullspace Expansion Theorem ($\dim\ker H^\star \ge 2$), we show that while physical degeneracy causes rank expansion in IPMs, global cost tightness remains strictly preserved. Combining observability-aware tangent remapping with closed-form multiplier recovery and an adaptive threshold $\tau_{\text{cert}}$, RASER provides efficient global optimality verification without direct SDP solving.

\bibliography{icra}
\bibliographystyle{IEEEtran}

\end{document}